\RequirePackage{amsmath}
\RequirePackage{fix-cm}
\documentclass[twocolumn,numbook,final]{svjour3}
\smartqed
\usepackage{graphicx}
\usepackage{mathptmx}      % use Times fonts if available on your TeX system

\usepackage{latexsym}
\usepackage{amsfonts}
\usepackage{bbold}
\usepackage{arydshln}
\usepackage{rotating}
\usepackage{cite}
\usepackage{physics}
\usepackage{xcolor}
\usepackage[hyphens]{url}
\usepackage[hidelinks]{hyperref}
\hypersetup{breaklinks=true}
\usepackage[misc]{ifsym}
\usepackage{pgf}

\usepackage{tikz}
\usepackage{mathdots}
\usepackage{yhmath}
\usepackage{cancel}
\usepackage{color}

\newcommand\inner[2]{\langle #1, #2 \rangle}
\DeclareMathOperator*{\argmax}{arg\,max}
\DeclareMathOperator*{\argmin}{arg\,min}
\DeclareMathOperator*{\prox}{\operatorname{Prox}}

\usepackage{etoolbox}
\newcommand*{\affmark}[1][*]{\textsuperscript{#1}}

\journalname{Journal of Mathematical Imaging and Vision (JMIV)}

\begin{document}

\title{Regularization by architecture: A deep prior approach for inverse problems
%\thanks{Grants or other notes
%about the article that should go on the front page should be
%placed here. General acknowledgments should be placed at the end of the article.}
}

\author{S{\"o}ren Dittmer\affmark[1,*] \and
        Tobias Kluth\affmark[1,*] \and
        Peter Maass\affmark[1,*] \and
        Daniel Otero Baguer\affmark[1,*]
}

\authorrunning{Dittmer et al.} % if too long for running head

\institute{
    \Letter\hspace{2mm} S{\"o}ren Dittmer\\
    \phantom{\Letter}\hspace{2mm} \href{mailto:sdittmer@math.uni-bremen.de}{sdittmer@math.uni-bremen.de}\\
    \\
    \phantom{\Letter}\hspace{2mm} Tobias Kluth,\\
    \phantom{\Letter}\hspace{2mm} \href{mailto:tkluth@math.uni-bremen.de}{tkluth@math.uni-bremen.de}\\
    \\
    \phantom{\Letter}\hspace{2mm} Peter Maass\\
    \phantom{\Letter}\hspace{2mm} \href{mailto:pmaass@math.uni-bremen.de}{pmaass@math.uni-bremen.de}\\
    \\
    \phantom{\Letter}\hspace{2mm} Daniel Otero Baguer\\
    \phantom{\Letter}\hspace{2mm} \href{mailto:otero@math.uni-bremen.de}{otero@math.uni-bremen.de}\\
    \\
    \affmark[1]\hspace{2mm} Center for Industrial Mathematics (ZeTeM),\\
    \phantom{\affmark[1]}\hspace{2mm} University of Bremen, Germany\\
    \\
    \affmark[*]\hspace{2mm} Equal contribution\\
    \\
    The code of the analytic deep prior implementation is available at \url{https://github.com/oterobaguer/analytic-deep-prior}\\
}

%     first address \\
%    Tel.: +123-45-678910\\
%    Fax: +123-45-678910\\
%    \email{fauthor@math.uni-bremen.de}           %  \\
    %             \emph{Present address:} of F. Author  %  if needed
%}

\date{Received: 30 November 2018 / Accepted: 18 October 2019}
% The correct dates will be entered by the editor

\maketitle
\begin{abstract}
The present paper studies so-called deep image prior (DIP) techniques in the context of ill-posed inverse problems. DIP networks have been recently in\-tro\-duced for applications in image processing; also first experimental results for applying DIP to inverse problems have been reported. This paper aims at discussing different interpretations of DIP and to obtain analytic results for specific network designs and linear operators. The main contribution is to introduce the idea of viewing these approaches as the optimization of Tikhonov functionals rather than optimizing net\-works. Besides theoretical results, we present numerical verifications. %The reconstructions obtained by deep prior net\-works are compared with state of the art methods.%for an academic example (integration operator) as well as for the inverse problem of magnetic particle imaging (MPI). The reconstructions obtained by deep prior net\-works are compared with state of the art methods.
\end{abstract}

\section{Introduction}
\label{intro}
Deep image priors (DIP)  were recently introduced in deep learning for some tasks in image processing \cite{ulyanov2017dip}. Usually, deep learning approaches to inverse problems proceed in two steps. In a first step (training) the parameters $\Theta$ of the deep neural network $\varphi_\Theta$  are optimized by minimizing a suitable loss function using large sets of training data. In a second step (application), new data is fed into the network for solving the desired task.

DIP approaches are radically different; they are based on unsupervised training using only a single data point $y^\delta$. More precisely, in the context of inverse problems, where we aim at solving ill-posed operator equations $Ax \sim y^\delta$, the task of DIP is to train a network $\varphi_\Theta(z)$ with parameters $\Theta$ by minimizing the simple loss function
\begin{equation} 
\min_\Theta \| A \varphi_\Theta(z) - y^\delta \|^2.
\label{basicloss}
\end{equation}
The minimization is with respect to $\Theta,$ the random input $z$ is kept fixed. After training the solution to the inverse problem is approximated directly by $\hat x = \varphi_\Theta(z).$ 

In image processing, common choices for $A$ are the identity operator (denoising) or a projection operator to a subset of the image domain (inpainting). For these applications, it has been observed, that minimizing the functional iteratively by gradient descent methods in combination with a suitable stopping criterion leads to amazing results \cite{ulyanov2017dip}.

Training with a single data point is the most striking property, which separates DIP from other neural network concepts. One might argue that the astonishing results \cite{ulyanov2017dip,mataev2019deepred,Cheng_2019_CVPR,van2018compressed} are only possible if the network architecture is fine-tuned to the specific task. This is true for obtaining optimal performance; nevertheless, the presented numerical results perform well even with somewhat generic network architectures such as autoencoders.

We are interested in analyzing DIP ap\-proach\-es for solving ill-posed inverse problems. As a side remark, we note that the applications (denoising, inpainting) mentioned above are modeled by either identity or projection operators, which are not ill-posed in the functional analytical setting \cite{louis,engl,rieder}. Typical examples of ill-posed inverse problems correspond to compact linear operators such as a large variety of tomographic measurement operators or parameter-to-state mappings for partial differential equations.

We aim at analyzing a specific network architecture $\varphi_\Theta$ and at interpreting the resulting DIP approach as a regularization technique in the functional analytical setting, and also at proving convergence properties for the minimizers of (\ref{basicloss}). In particular, we are interested in network architectures, which themselves can be interpreted as a minimization algorithm that solves a regularized inverse problem of the form
\begin{equation}
x(B) = \argmin_x\frac{1}{2} \| B x - y^\delta\|^2  + \alpha R(x),
\end{equation}
where $R$ is a given convex function and $B$ a learned operator.

In general, deep learning approaches for inverse problems have their own characteristics, and naive applications of neural networks can fail for even the most simple inverse problems, see \cite{maass2018trivial}. However, there is a growing number of compelling numerical experiments using suitable network designs for some of the toughest inverse problems such as photo-acou\-stic tomography \cite{hauptmann2018model} or X-ray tomography with very few measurements \cite{jin2017deep,adler2018learned}. Concerning networks based on deep prior approaches for inverse problems, first experimental investigations have been re\-por\-ted, see \cite{ulyanov2017dip, van2018compressed, mataev2019deepred}. Similar as for the above-mentioned tasks in image processing, DIPs for inverse problems rely on two ingredients:
\begin{enumerate}
    \item A suitable network design, which leads to our phrase ``regularization by architecture''.
    \item Training algorithms for iteratively minimizing (\ref{basicloss}) with respect to $\Theta$ in combination with a suitable stopping criterion.
\end{enumerate}

In this paper, we present different mathematical interpretations of DIP approaches, and we analyze two network designs in the context of inverse problems in more detail. It is organized as follows: In Section 2, we discuss some relations to existing results and make a short survey of the related literature. In Section 3, we then state different interpretations of DIP approaches and the network architectures that we use, as a basis for the subsequent analysis. We start with a first mathematical result for a trivial network design, which yields a connection to Landweber iterations. We then consider a fully connected feedforward network with $L$ identical layers, which generates a proximal gradient descent for a modified Tikhonov functional. In Section 4, we use this last connection to define the notion of analytic deep prior networks, for which one can strictly analyze its regularization and convergence properties. The key to the theoretical findings is a change of view, which allows for the interpretation of DIP approaches as optimizing families of Tikhonov functionals. Finally, we exemplify our theoretical findings with numerical examples for the standard linear integration operator.

\section{Deep prior and related research}

We start with a description of general deep prior concepts. Afterwards, we address similarities and differences to other approaches, such as LISTA \cite{lista}, in more detail.

\subsection{The deep prior approach}
Present results on deep prior networks utilize feedforward architectures.
In general, a feedforward neural network is an algorithm that starts with input $x^0=z$, computes iteratively
\begin{equation*}
x^{k+1} = \phi \left(W_k x^k + b_k  \right)
\end{equation*}
$\mbox{for}\ \  k=0,..,L-1$ and outputs
\begin{equation*}
\varphi_\Theta(z)= x^L \ .
\end{equation*}
The parameters of this system are denoted by
\begin{equation*}
\Theta = \left\{W_0,..,W_{L-1}, b_0,..,b_{L-1}  \right\}
\end{equation*} and $\phi$ denotes a non-linear activation function.

In order to highlight one of the unique features of deep image priors, let us first refer to classical generative networks that require training on large data sets.

In this classical setting we are given an operator $A: X \to Y$ between Hilbert spaces $X,Y$, as well as a set of training data $(x_i,y_i^\delta)$, where $y_i^\delta$ is a noisy version of $Ax_i$ satisfying $\|y_i^\delta - Ax_i\| \le \delta$.
Here the usual deep learning approach is to use a network for direct inversion and the parameters $\Theta$ of the network are obtained by minimizing the loss function
\begin{equation}
\label{classicloss}
\min_\Theta \sum_{i=1}^N\ \|\varphi_\Theta(y_i^\delta) - x_i \|^2\ .
\end{equation}
After training $\Theta$ is fixed and the network is used to approximate the solution of the inverse problem with new data $y^\delta$ by computing
$x=\varphi_\Theta(y^\delta)$.
For a recent survey on this approach and more general deep learning concepts for inverse problems see~\cite{arridge2019solving}.

In general, this approach relies on the underlying assumption, that complex distributions of suitable solutions $x$, e.g., the distribution of natural images, can be approximated by neural networks \cite{fista, unser2007fista, bruna2013invariant}. The parameters $\Theta$ are trained for the specific distribution of training data and are fixed after training.
One then  expects, that choosing  a new data set as input, i.e., $z=y^\delta$ will generate a suitable solution to $Ax \sim y^\delta$ ~\cite{bora2017compressed}. 
Hence, after training the distribution of solutions is parametrized by the inputs $z$.

In contrast, DIP is an unsupervised approach using only a single data point for training. That means, for given data $y^\delta$ and fixed $z$, the parameters $\Theta$ of the network $\varphi_\Theta$ are obtained by minimizing the loss function (\ref{basicloss}). The solution to the inverse problem is then denoted by $\hat x = \varphi_\Theta(z)$. Hence, deep image priors keep $z$ fixed and aim at parameterizing the solution with $\Theta$. It has been observed in several works \cite{ulyanov2017dip, van2018compressed, mataev2019deepred,Cheng_2019_CVPR} that this approach indeed leads to remarkable results for problems such as inpainting or denoising. 

To some extent, the success of deep image priors is rooted in the careful design of network architectures. For example, \ \cite{ulyanov2017dip} uses a U-Net-like “hourglass” architecture with skip connections, and the amazing results show that such an architecture implicitly captures some statistics of natural images. However, in general, the DIP learning process may converge towards noisy images or undesirable reconstructions. The whole success relies on a combination of the architecture with a suitable optimization method and stopping criterion. Nevertheless, the authors claim the architecture has a positive impact on the exploration of the solution space during the iterative optimization of $\Theta$. They show that the training process descends quickly to ``natural-looking'' images but requires much more steps to produce noisy images. This is also supported by the theoretical results of~\cite{saxe2013exact} and the observations of ~\cite{zhang2016understanding},  which shows that deep networks can fit noise very well but need more training time to do so. Another paper that hints in this direction is~\cite{anonymous2019on}, which analyzes whether neural networks could have a bias towards approximating low frequencies.

There are already quite a few works that deal with deep prior approaches. Following, we mention the most relevant ones to our work. The original deep image prior article \cite{ulyanov2017dip} introduces the DIP concept and presents experimental evidence that today's network architectures are in and of themselves conducive to image reconstruction. Another work \cite{van2018compressed} explores the applicability of DIP to problems in compressed sensing. Also, \cite{mataev2019deepred} discusses how to combine DIP with the regularization by denoising approach and \cite{Cheng_2019_CVPR} explores DIP in the context of stationary Gaussian processes. All of these introduce and discuss variants of DIP concepts; however, neither of them addresses the intrinsic regularizing properties of the network concerning ill-posed inverse problems.

\subsection{Deep prior and unrolled proximal gradient architectures}
A major part of this paper is devoted to analyzing the DIP approach in combination with an unrolled proximal gradient network $\varphi_\Theta$. Hence, there is a natural connection to the well-established analysis of LISTA schemes. Before we sketch the state of research in this field, we highlight the two major differences (loss function, training data) to the present approach. LISTA is based on a supervised training using multiple data points $(x_i, y_i^\delta)$, $i=1,..N$ where $y_i^\delta$ is a noisy representation of $Ax_i$. The loss function is \eqref{classicloss}. DIP, however, is based on unsupervised learning using the loss function \eqref{basicloss} and a single data point $y^\delta$. Hence, DIP with the unrolled proximal gradient network shares the architecture with LISTA, but its concept, as well as its analytic properties, are different. Nevertheless, the analysis we will present in Section~\ref{sec:dip_tikhonov} will exhibit structures similar to the ones appearing in the LISTA-related literature. Hence we shortly review the major contributions in this field.

Similarities are most visible when considering algorithms and convergence analysis for sparse coding applications \cite{xin2016maximal, moreau2016understanding, sulam2019multi, liu2018alista, sprechmann2015learning}. The field of sparse coding makes heavy use of proximal splitting algorithms and, since the advent of LISTA, of trained architectures inspired by truncated versions of these algorithms. In the broadest sense, all of these methods are expressions of ``Learning to learn by gradient descent'' \cite{andrychowicz2016learning}. Once more, we would like to emphasize that these results utilize multiple data points while DIP does not require any training data but only one measurement. Another key difference is that we approach the topic from an ill-posed inverse problem perspective, which (a) grounds our approach in the functional analytic realm and (b) considers ill-posed (not only ill-conditioned) problems in the Nashed sense, i.e., allows the treatment of unstable inverses \cite{engl}. These two points fundamentally differentiate the present approach from traditional compressed sensing considerations which usually deal with (a) finite dimensional formulations and (b) forward operators given by well-conditioned, carefully hand-crafted settings or dictionaries, which are optimized using large sets of training data \cite{sprechmann2015learning}.

Coming back to LISTA for sparse coding applications, there are many excellent papers \cite{moreau2016understanding,giryes2018tradeoffs,meinhardt2017learning} which are devoted to a strict mathematical analysis of different aspects of LISTA-like approaches. In \cite{moreau2016understanding}, the authors show under which conditions sparse coding can benefit from LISTA-like trained structures and asks how good trained sparsity estimators can be, given a computational budget. The article \cite{giryes2018tradeoffs} deals with a similar trade-off proposing the quite exciting, ``inexact proximal gradient descent''. The paper \cite{chen2018theoretical} proposes, based on theoretically founded considerations, a sibling architecture to LISTA. Moreover,  \cite{papyan2018theoretical} argues that deep learning architectures, in general, can be interpreted as multi-stage proximal splitting algorithms.

Finally, we want to point at publications, which address deep learning with only a few data points for training, see, e.g., \cite{joergluecke2018} and the references therein. However, they do not address the architectures relevant for our publication, and they do not refer to the specific complications of inverse problems.

\section{Deep prior architectures and interpretations}
\label{interpretations}
In this section, we discuss different perspectives on deep prior networks, which open the path to provable mathematical results.
The first two subsections are devoted to special network architectures, and the last two subsections deal with more general points of view.
 
\subsection{A trivial architecture} 
\label{trivial}
We aim at solving ill-posed inverse problems. For a given operator $A,$ the general task in inverse problems is to recover an approximation for $x^\dagger$ from measured noisy data
\[y^\delta = A x^\dagger + \tau,\]
where $\tau$, with $\|\tau\| \le \delta,$ describes the noise in the measurement.

The deep image prior approach to inverse problems asks to train a network $\varphi_\Theta(z)$ with parameters $\Theta$ and fixed input $z$ by minimizing \ $\| A \varphi_\Theta(z) - y^\delta\|^2$ with an optimization method such as gradient descent with early stopping. After training, a final run of the network computes $\hat x = \varphi_\Theta(z)$ as an approximation to $x^\dagger$.

We consider a trivial single-layer network without activation function, see Figure \ref{fig:simple_network}. This network simply outputs $\Theta,$ i.e., $\varphi_\Theta(z)=\Theta$. In this case, the network parameter $\Theta$ is a vector, which is chosen to have the same dimension as $x$. That means, that training the network by gradient descent of $\| A \varphi_\Theta(z) - y^\delta\|^2 = \| A \Theta - y^\delta\|^2$ with respect to $\Theta$ is equivalent to the classical Landweber iteration, which is a gradient descent method for $\| A x - y^\delta\|^2$ with respect to $x$.

Landweber iterations are slowly converging. However, in combination with a suitable stopping rule, they are optimal regularization schemes for diminishing noise level $\delta \rightarrow 0$, \cite{louis,engl,rieder}. Despite the apparent trivialization of the neural network approach, this shows that there is potential in training such networks with a single data point for solving ill-posed inverse problems.

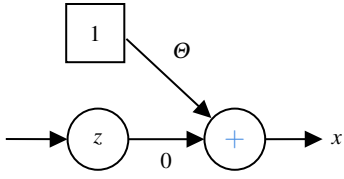
\begin{figure}
    \centering
    \tikzset{every picture/.style={line width=0.75pt}} %set default line width to 0.75pt        

\begin{tikzpicture}[x=0.75pt,y=0.75pt,yscale=-1,xscale=1]
%uncomment if require: \path (0,413.6363525390625); %set diagram left start at 0, and has height of 413.6363525390625

%Shape: Ellipse [id:dp8880299466193733] 
\draw   (102.2,150.18) .. controls (102.2,141.68) and (108.99,134.79) .. (117.36,134.79) .. controls (125.73,134.79) and (132.52,141.68) .. (132.52,150.18) .. controls (132.52,158.68) and (125.73,165.57) .. (117.36,165.57) .. controls (108.99,165.57) and (102.2,158.68) .. (102.2,150.18) -- cycle ;
%Straight Lines [id:da8019606440081755] 
\draw    (132.52,150.18) -- (168.52,150.18) ;
\draw [shift={(170.52,150.18)}, rotate = 180] [fill={rgb, 255:red, 0; green, 0; blue, 0 }  ][line width=0.75]  [draw opacity=0] (8.93,-4.29) -- (0,0) -- (8.93,4.29) -- cycle    ;

%Straight Lines [id:da1476911779757435] 
\draw    (71.36,150.09) -- (100.2,150.17) ;
\draw [shift={(102.2,150.18)}, rotate = 180.17] [fill={rgb, 255:red, 0; green, 0; blue, 0 }  ][line width=0.75]  [draw opacity=0] (8.93,-4.29) -- (0,0) -- (8.93,4.29) -- cycle    ;

%Straight Lines [id:da40590575799960327] 
\draw    (131.55,99.73) -- (172.54,138.18) ;
\draw [shift={(174,139.55)}, rotate = 223.16] [fill={rgb, 255:red, 0; green, 0; blue, 0 }  ][line width=0.75]  [draw opacity=0] (8.93,-4.29) -- (0,0) -- (8.93,4.29) -- cycle    ;

%Straight Lines [id:da9293591475689398] 
\draw    (200.84,150.18) -- (227.52,150.14) ;
\draw [shift={(229.52,150.14)}, rotate = 539.9100000000001] [fill={rgb, 255:red, 0; green, 0; blue, 0 }  ][line width=0.75]  [draw opacity=0] (8.93,-4.29) -- (0,0) -- (8.93,4.29) -- cycle    ;

%Shape: Ellipse [id:dp7095470205002026] 
\draw   (170.52,150.18) .. controls (170.52,141.68) and (177.31,134.79) .. (185.68,134.79) .. controls (194.05,134.79) and (200.84,141.68) .. (200.84,150.18) .. controls (200.84,158.68) and (194.05,165.57) .. (185.68,165.57) .. controls (177.31,165.57) and (170.52,158.68) .. (170.52,150.18) -- cycle ;
%Shape: Rectangle [id:dp20988953075737005] 
\draw   (101.51,82.17) -- (130.52,82.17) -- (130.52,111.45) -- (101.51,111.45) -- cycle ;

% Text Node
\draw (117.36,150.18) node  [align=left] {$\displaystyle z$};
% Text Node
\draw (185.68,150.18) node [scale=1.2,color={rgb, 255:red, 74; green, 144; blue, 226 }  ,opacity=1 ] [align=left] {$\displaystyle \textcolor[rgb]{0.29,0.56,0.89}{+}$};
% Text Node
\draw (116.02,96.81) node  [align=left] {$\displaystyle 1$};
% Text Node
\draw (159.46,105.51) node  [align=left] {$\displaystyle \Theta$};
% Text Node
\draw (151.46,160.85) node  [align=left] {$\displaystyle 0$};
% Text Node
\draw (137.46,169.85) node  [align=left] {$ $};
% Text Node
\draw (236.68,150.18) node  [align=left] {$\displaystyle x$};

\end{tikzpicture}
    \caption{A simple network with scalar input, a single layer and no activation function. For any arbitrary input $z$  one obtains $\varphi_\Theta(z) = \Theta$}
    \label{fig:simple_network}
\end{figure}

\subsection{Unrolled proximal gradient architecture}
\label{sec:unrolled_architecture}
In this section, we aim at rephrasing DIP, i.e., the minimization of (\ref{basicloss}) with respect to $\Theta$, as an approach for learning optimized Tikhonov functionals for inverse problems. This change of view, i.e.,\ regarding deep inverse priors as optimization of functionals rather than networks, opens the way for analytic investigations in Section \ref{sec:dip_tikhonov}.

We use the particular architecture, which was introduced in \cite{lista}, i.e. a fully connected feedforward network with $L$ layers of identical size,
\begin{equation}
\varphi_\Theta(z)= x^L,
\label{eq:network_1}
\end{equation}
where 
\begin{equation}
x^{k+1} = \phi \left(W x^k + b  \right)
\label{eq:network_2}
\end{equation}
The affine linear map  $\Theta = (W,b)$ is the same for all layers.  
The matrix $W$ is restricted to obey  $I-W = \lambda B^* B$ ($I$ denotes the identity operator) for some $B$ and the bias is determined via $b = \lambda B^* y^\delta$, see Figure \ref{fig:lista_network}.
If the activation function of the network is chosen as the proximal mapping of a regularizing functional $\lambda \alpha R$, then $\varphi_\Theta(z)$
is identical to the $L$-th iterate of a proximal gradient descent method for minimizing
\begin{equation}
J_B(x)= \frac{1}{2} \| B x - y^\delta\|^2  + \alpha R(x),
\label{eq:constraint_functional}
\end{equation}
see \cite{daubechies2004surrogate} or Appendix I.

\begin{remark}
Restricting activation functions to be proximal mappings is not as severe as it might look at first glance. E.g., \ ReLU is the proximal mapping for the indicator function of positive real numbers, and soft shrinkage is the proximal mapping for the modulus function.
\end{remark}

This allows the interpretation that every weight update, i.e., every gradient step for minimizing \eqref{basicloss} with respect to $\Theta$ or $B$, changes the functional $J_B$. Hence, DIP can be regarded as optimizing a functional, which in-turn is minimized by the network. This view is the starting point for investigating convergence properties in Section \ref{sec:dip_tikhonov}.

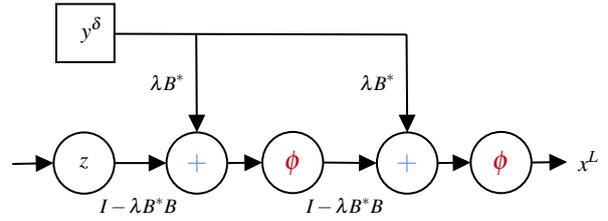
\begin{figure}
    \centering
    \tikzset{every picture/.style={line width=0.75pt}} %set default line width to 0.75pt        

\begin{tikzpicture}[x=0.75pt,y=0.75pt,yscale=-1,xscale=1]
%uncomment if require: \path (0,413.6363525390625); %set diagram left start at 0, and has height of 413.6363525390625

%Shape: Ellipse [id:dp0352016737170977] 
\draw   (33.2,359.18) .. controls (33.2,350.68) and (39.99,343.79) .. (48.36,343.79) .. controls (56.73,343.79) and (63.52,350.68) .. (63.52,359.18) .. controls (63.52,367.68) and (56.73,374.57) .. (48.36,374.57) .. controls (39.99,374.57) and (33.2,367.68) .. (33.2,359.18) -- cycle ;
%Straight Lines [id:da28220817114220664] 
\draw    (63.52,359.18) -- (87.52,359.18) ;
\draw [shift={(89.52,359.18)}, rotate = 180] [fill={rgb, 255:red, 0; green, 0; blue, 0 }  ][line width=0.75]  [draw opacity=0] (8.93,-4.29) -- (0,0) -- (8.93,4.29) -- cycle    ;

%Straight Lines [id:da2898293785080168] 
\draw    (12.55,359.64) -- (31.2,359.22) ;
\draw [shift={(33.2,359.18)}, rotate = 538.73] [fill={rgb, 255:red, 0; green, 0; blue, 0 }  ][line width=0.75]  [draw opacity=0] (8.93,-4.29) -- (0,0) -- (8.93,4.29) -- cycle    ;

%Straight Lines [id:da8297452732925297] 
\draw    (104.35,294.55) -- (104.67,341.79) ;
\draw [shift={(104.68,343.79)}, rotate = 269.61] [fill={rgb, 255:red, 0; green, 0; blue, 0 }  ][line width=0.75]  [draw opacity=0] (8.93,-4.29) -- (0,0) -- (8.93,4.29) -- cycle    ;

%Straight Lines [id:da36936160961040176] 
\draw    (119.84,359.18) -- (135.37,358.94) ;
\draw [shift={(137.37,358.91)}, rotate = 539.11] [fill={rgb, 255:red, 0; green, 0; blue, 0 }  ][line width=0.75]  [draw opacity=0] (8.93,-4.29) -- (0,0) -- (8.93,4.29) -- cycle    ;

%Shape: Ellipse [id:dp9420272688378455] 
\draw   (89.52,359.18) .. controls (89.52,350.68) and (96.31,343.79) .. (104.68,343.79) .. controls (113.05,343.79) and (119.84,350.68) .. (119.84,359.18) .. controls (119.84,367.68) and (113.05,374.57) .. (104.68,374.57) .. controls (96.31,374.57) and (89.52,367.68) .. (89.52,359.18) -- cycle ;
%Straight Lines [id:da3855896253840876] 
\draw    (63.52,294.18) -- (209.35,294.55) ;

%Shape: Ellipse [id:dp906108137696755] 
\draw   (137.37,358.91) .. controls (137.37,350.41) and (144.16,343.52) .. (152.53,343.52) .. controls (160.91,343.52) and (167.69,350.41) .. (167.69,358.91) .. controls (167.69,367.41) and (160.91,374.3) .. (152.53,374.3) .. controls (144.16,374.3) and (137.37,367.41) .. (137.37,358.91) -- cycle ;
%Straight Lines [id:da1286111859404977] 
\draw    (167.69,358.91) -- (192.52,359.16) ;
\draw [shift={(194.52,359.18)}, rotate = 180.58] [fill={rgb, 255:red, 0; green, 0; blue, 0 }  ][line width=0.75]  [draw opacity=0] (8.93,-4.29) -- (0,0) -- (8.93,4.29) -- cycle    ;

%Straight Lines [id:da2598247826110667] 
\draw    (209.35,294.55) -- (209.67,341.79) ;
\draw [shift={(209.68,343.79)}, rotate = 269.61] [fill={rgb, 255:red, 0; green, 0; blue, 0 }  ][line width=0.75]  [draw opacity=0] (8.93,-4.29) -- (0,0) -- (8.93,4.29) -- cycle    ;

%Straight Lines [id:da7109648978488654] 
\draw    (224.84,359.18) -- (239.37,358.94) ;
\draw [shift={(241.37,358.91)}, rotate = 539.06] [fill={rgb, 255:red, 0; green, 0; blue, 0 }  ][line width=0.75]  [draw opacity=0] (8.93,-4.29) -- (0,0) -- (8.93,4.29) -- cycle    ;

%Shape: Ellipse [id:dp9060795397790953] 
\draw   (194.52,359.18) .. controls (194.52,350.68) and (201.31,343.79) .. (209.68,343.79) .. controls (218.05,343.79) and (224.84,350.68) .. (224.84,359.18) .. controls (224.84,367.68) and (218.05,374.57) .. (209.68,374.57) .. controls (201.31,374.57) and (194.52,367.68) .. (194.52,359.18) -- cycle ;
%Shape: Ellipse [id:dp4824528891899653] 
\draw   (241.37,358.91) .. controls (241.37,350.41) and (248.16,343.52) .. (256.53,343.52) .. controls (264.91,343.52) and (271.69,350.41) .. (271.69,358.91) .. controls (271.69,367.41) and (264.91,374.3) .. (256.53,374.3) .. controls (248.16,374.3) and (241.37,367.41) .. (241.37,358.91) -- cycle ;
%Straight Lines [id:da5213849732970888] 
\draw    (271.69,358.91) -- (287.55,358.67) ;
\draw [shift={(289.55,358.64)}, rotate = 539.13] [fill={rgb, 255:red, 0; green, 0; blue, 0 }  ][line width=0.75]  [draw opacity=0] (8.93,-4.29) -- (0,0) -- (8.93,4.29) -- cycle    ;

%Shape: Rectangle [id:dp06580215784435728] 
\draw   (34.51,279.17) -- (63.52,279.17) -- (63.52,308.45) -- (34.51,308.45) -- cycle ;

% Text Node
\draw (48.36,359.18) node  [align=left] {$\displaystyle z$};
% Text Node
\draw (104.68,359.18) node [scale=1.2,color={rgb, 255:red, 74; green, 144; blue, 226 }  ,opacity=1 ] [align=left] {$\displaystyle \textcolor[rgb]{0.29,0.56,0.89}{+}$};
% Text Node
\draw (89.96,318.51) node [scale=0.9] [align=left] {$\displaystyle \lambda B^{*}$};
% Text Node
\draw (75.46,380.85) node [scale=0.9] [align=left] {$\displaystyle I-\lambda B^{*} B$};
% Text Node
\draw (52.36,292.18) node  [align=left] {$\displaystyle y^{\delta }$};
% Text Node
\draw (152.53,358.91) node [scale=1.2,color={rgb, 255:red, 208; green, 2; blue, 27 }  ,opacity=1 ] [align=left] {$\displaystyle \phi $};
% Text Node
\draw (209.68,359.18) node [scale=1.2,] [align=left] {$\displaystyle \textcolor[rgb]{0.29,0.56,0.89}{+}$};
% Text Node
\draw (256.53,358.91) node [scale=1.2,color={rgb, 255:red, 208; green, 2; blue, 27 }  ,opacity=1 ] [align=left] {$\displaystyle \phi $};
% Text Node
\draw (300.36,358.18) node  [align=left] {$\displaystyle x^{L}$};
% Text Node
\draw (178.46,380.85) node [scale=0.9] [align=left] {$\displaystyle I-\lambda B^{*} B$};
% Text Node
\draw (124.46,384.85) node  [align=left] {$ $};
% Text Node
\draw (194.96,318.51) node [scale=0.9] [align=left] {$\displaystyle \lambda B^{*}$};

\end{tikzpicture}
    \caption{Unrolled proximal gradient network with $L=2$.}
    \label{fig:lista_network}
\end{figure}

\subsection{Two perspectives based on regression}
\label{perspectives}
The following subsections address more general concepts, which open the way to further analytic investigations, which, however, are not considered further in this paper.
The reader interested in the regularization properties for DIP approaches for inverse problems only may jump directly to Section 4.

In this subsection we present two different perspectives on solving inverse problems with the DIP via the minimization of a functional as discussed in the subsection above. The first perspective is based on a reinterpretation of the minimization of the functional (\ref{basicloss}) in the finite, real setting, i.e. $A\in\mathbb{R}^{m\times n}$. This setting allows us to write
\begin{align}
\min_\Theta\|A\varphi_\Theta(z)-y^\delta\|^2 & = \min_{x\in\mathcal{R}(\varphi_\cdot(z))} \|Ax-y^\delta\|^2\\
& = \min_{x\in\mathcal{R}(\varphi_\cdot(z))} \sum_{i=1}^m (x^*a_i-y^\delta_i)^2,
\end{align}
where $\mathcal{R}(\varphi_\cdot(z))$ denotes the range of the network with regard to $\Theta$ for a fixed $z$ and $a_i$ the rows of the matrix $A$ as well as $y^\delta_i$ the entries of the vector $y^\delta$. This setting allows for the interpretation that we are solving a linear regression, parameterized by $x$, which is constrained by a deep learning hypothesis space and given by data pairs of the form $(a_i, y^\delta_i)$.

The second perspective is based on the rewriting of the optimization problem via the method of Lagrange multipliers. We start by considering the constrained optimization problem
\begin{equation}
\min_{x\in X, \Theta}\|Ax-y^\delta\|^2 \text{ s.t. } \|x-\varphi_\Theta(z)\|^2=0.
\end{equation}
If we now assume that $\varphi$ has continuous first partial derivatives with regard to $\Theta$, the Lagrange functional
\begin{equation}
\mathcal{L}(\Theta,x,\lambda) = \|Ax-y^\delta\|^2 + \lambda \|x-\varphi_\Theta(z)\|^2,
\end{equation}
with the correct Lagrange multiplier $\lambda=\lambda_0$, has a stationary point at each minimum of the original constraint optimization problem. This gives us a direct connection to unconstrained variational approaches like Tikhonov functionals.
 
 \subsection{The Bayesian point of view}
The Bayesian approach to inverse problems focuses on computing MAP (maximum a posteriori probability) estimators, i.e.\ one aims for
\begin{equation}
\hat x = \argmax_{x\in X}p(x|y^\delta),
\end{equation}
where $p:X\times Y\to\mathbb{R}_+\cup \{0\}$ is a conditional PDF. From standard Bayesian theory we obtain  
\begin{equation}
    \hat x 
    = \argmin_{x\in X}  \left\{ -\log[p(y^\delta|x)]-\log[p(x)] \right\} \ .
\end{equation}
The setting for  inverse problems, i.e.\ $Ax+\tau = y^\delta $ with
$\tau\sim\mbox{Normal}(0,\sigma^2\mathbb{1}_Y)$, yields ($\lambda=2\sigma^2$)
\begin{align*}
    \hat x 
    &=: \argmin_{x\in X} \|Ax-y^\delta\|^2-\lambda \log[p(x)] \ .
\end{align*}
We now decompose $x$ into  $x_\perp := P_{\mathcal{N}(A)^\perp}(x)$,\, $ and $\, $x_\mathcal{N} := P_{\mathcal{N}(A)}(x)$, where $\mathcal{N}(A)$ denotes the nullspace of $A$ and where $P_{\mathcal{N}(A)}(x)$, resp. $P_{\mathcal{N}(A)^\perp}(x)$, denotes the orthogonal projection onto $\mathcal{N}(A)$, resp. $\mathcal{N}(A)^\perp$.
Setting $\hat x = (x_\mathcal{N}, x_\perp)$ yields
\begin{align*}
    \hat x &= \argmin_{x\in X} \|Ax_\perp-y^\delta\|^2-\lambda \log\ p(x_\mathcal{N}, x_\perp) \\
     &= \argmin_{x\in X} \|Ax_\perp-y^\delta\|^2
        -\lambda \log\ p(x_\perp) - \lambda \log\ p(x_\mathcal{N}|x_\perp).\\
     &= \argmin_{x\in X}\ \rlap{$\overbrace{\phantom{\|Ax_\perp-y^\delta\|^2-\lambda \log\ p(x_\perp)}}^{(I)}$}\|Ax_\perp-y^\delta\|^2
     \underbrace{-\lambda \log\ p(x_\perp) - \lambda \log\ p(x_\mathcal{N}|x_\perp)}_{(II)}.
\end{align*}
The data $y^\delta$ only contains information about $x_\perp$,
which in classical regularization is exploited by restricting any reconstruction to  ${\cal N}(A)^\perp$.

However, if available, $p(x_\mathcal{N}|x_\perp)$ is a measure on how to extend $x_\perp$ with an $x_\perp \in {\cal N}(A)^\perp$ to a suitable  $x = (x_\mathcal{N}, x_\perp)$. The classical regularization of inverse problems uses the trivial extension by zero, i.e., \ $x = (0, x_\perp)$, which is not necessarily optimal. If we accept the interpretation that a network can be a meaningful parametrization of the set of suitable solutions $x$, then $p(x) \equiv 0$ for all $x$ not in the range of the network and optimizing the network will indeed yield a non-trivial completion $x = (x_\mathcal{N}, x_\perp)$. More precisely (I) can be interpreted to be a deep prior on the measurement and (II) to be a deep prior on the nullspace part of the problem.

%%%%%%%%%%%%%%%%%%%%%%%%%%%%%%%%%%%%%%

\section{Deep priors and Tikhonov functionals}
\label{sec:dip_tikhonov}
In this section, we consider the particular network architecture given by unrolled proximal gradient schemes, see Section \ref{sec:unrolled_architecture}. We aim at embedding this approach into the classical regularization theory for inverse problems. For a strict mathematical analysis, we will introduce the notion of an analytic deep prior network, which then allows interpreting the training of the deep prior network as an optimization of a Tikhonov functional. The main result of this section is Theorem \ref{th:order_optimal}, which states that analytic deep priors in combination with a suitable stopping rule are indeed order optimal regularization schemes.
Numerical experiments in Section \ref{sec:numerics} demonstrate that such deep prior approaches lead to smaller reconstruction errors when compared with standard Tikhonov reconstructions. The superiority of this approach can be proved, however, only for the rather unrealistic case, that the solution coincides with a singular function of $A$.

\subsection{Unrolled proximal gradient networks as deep priors for inverse problems}
In this section, we consider linear operators $A$ and aim at rephrasing DIP, i.e., the minimization of (\ref{basicloss}) with respect to $\Theta$,  as a constrained optimization problem. This change of view, i.e.,\ regarding deep inverse priors as an optimization of a simple but constrained functional, rather than networks, opens the way for analytic investigations. We will use an unrolled proximal gradient architecture for the network $\varphi_\Theta (z)$ in \eqref{basicloss}. The starting point for our investigation is the common observation, see  \cite{combettes2005splitting, lista} or Appendix~I, that an unrolled proximal gradient scheme as defined in Section~\ref{sec:unrolled_architecture} approximates a minimizer $x(B)$ of (\ref{eq:constraint_functional}). Assuming that a unique minimizer $x(B)$ exists as well as neglecting the difference between $x(B)$ and the approximation $\varphi_\Theta(z)$ achieved by the unrolled proximal gradient motivates the following definition of analytic deep priors.

\begin{definition}
Let us assume that measured data $y^\delta \in Y$,  a fixed  $\alpha>0$, a convex penalty functional $R:X\to\mathbb{R}$, and a measurement operator $A \in \mathcal{L}(X,Y)$ are given. We consider the minimization problem
\begin{align}
 \min_B F(B)= \min_B\frac{1}{2} \|A x(B) - y^\delta\|^2,
\label{eq:adip}
\end{align}
subject to the constraint
\begin{equation}
x(B) = \argmin_x J_B(x) = \argmin_x\frac{1}{2} \| B x - y^\delta\|^2  + \alpha R(x).
\end{equation}
We assume that for every $B \in \mathcal{L}(X,Y)$ there is a  unique minimizer $x(B)$. We call this constrained minimization problem an analytic deep prior and denote by $x(B)$ the resulting solution to the inverse problems posed by $A$ and $y^\delta$.
\label{def:adip}
\end{definition}

We can also use this technical definition as the starting point of our consideration and retrieve the  neural network architecture  by considering the following approach for solving the minimization problem stated in the above definition. Assuming that $R$ has a proximal operator, we can compute $x(B)$, given $B$, via proximal gradient method. I.e., via the
(for a suitable choice of $\lambda>0$ and an arbitrary $x^0=z\in X$) converging iteration
\begin{equation}
     x^{k+1} = \prox_{\lambda \alpha R} \left(x^k - \lambda B^*(Bx^k-y^\delta)\right).
\end{equation}
Following this iteration for $L$ steps can be seen as the forward pass of a particular architecture of a fully connected feed-forward network with $L$ layers of identical size as described in (\ref{eq:network_1}) and (\ref{eq:network_2}). The affine linear map given by $\Theta=(W,b)$ is the same for all layers. Moreover, the activation function of the network is given by the proximal mapping of $\lambda \alpha R$, the matrix $W$ is given via $I-W = \lambda B^* B$ ($I$ denotes the identity operator), and the bias is determined by $b = \lambda B^* y^\delta$.

From now on we will assume that the difference between $x^L$ and $x(B)$ is negligible, i.e.,
\begin{equation}
    \label{analytic_assumption}
    x^L = x(B).
\end{equation}

\begin{remark}
The task in the DIP approach is to find $\Theta$ (network parameters). Analogously, in the analytic deep prior, we try to find the operator $B$.
\end{remark}

We now examine the analytic deep image prior utilizing the proximal gradient descent approach to compute $x(B)$. Therefore we will focus on the minimization of (\ref{eq:adip})
with respect to $B$ for given data $y^\delta$ by means of gradient descent.

The stationary points are characterized by $\partial F(B)=0$ and gradient descent iterations with stepsize $\eta$ are given by
\begin{equation}
B^{\ell+1} = B^\ell - \eta \partial  F (B^\ell).
\end{equation}
Hence we need to compute the derivative of $F$ with respect to $B$.

\begin{lemma}
Consider an analytic deep prior with the proximal gradient descent approach as described above. We define
\begin{equation}
\psi(x,B) =  \prox_{\lambda \alpha R} \left(x - \lambda B^*(Bx-y^\delta)\right) - x.
\end{equation}
Then 
\begin{equation}
\partial F (B) = \partial x(B)^*A^*(Ax(B) - y^\delta)
\end{equation}
with
\begin{equation}
\partial x(B) = - \psi_x(x(B), B)^{-1} \psi_B(x(B),\, B),
\end{equation}
which leads to the gradient descent
\begin{equation}
B^{\ell+1}= B^\ell - \eta \partial  F (B^\ell).
\end{equation}
\label{lemma:derivative}
\end{lemma}

This lemma allows to obtain an explicit description of the gradient descent for $B$, which in turn leads to an iteration of functionals $J_B$ and minimizers $x(B)$. We will now exemplify this derivation for a rather academic example, which however highlights in particular the differences between a classical Tikhonov minimizer, i.e.\ 
\[x(A) = \argmin_x \frac{1}{2} \| A x - y^\delta\|^2  + \frac{\alpha}{2} \Vert x \Vert^2,\]
and the solution of the DIP approach.

\subsubsection{Example}
In this example we examine analytic deep priors for linear inverse problems $A:X \rightarrow Y$, i.e. $A, B \in {\cal L}(X,Y)$, and
\begin{equation}
    R(x)=\frac{1}{2}\| x\|^2.
\end{equation}
The rather abstract characterization of the previous section can be made explicit for this setting. Since $J_B(x)$ is the classical Tikhonov regularization, which can be solved by 
\begin{equation}
    x(B) = (B^*B+\alpha I)^{-1}B^*y^\delta,
\end{equation}
we can rewrite the analytic deep prior reconstruction as $x(B)$, where $B$ is minimizing
\begin{equation}
    F(B) = \frac{1}{2} \| A (B^*B+\alpha I)^{-1}B^*y^\delta - y^\delta\|^2.
    \label{eq:example_functional}
\end{equation}

\begin{lemma}
Following Lemma \ref{lemma:derivative}, assuming $B^0=A$ and computing one step of gradient descent to minimize the functional with respect to $B$, yields
\begin{equation}
    B_1 = A - \eta\partial F(A)
\end{equation}
with
\begin{align}
\partial F (A)  &=\partial x (A)^*A^*(Ax(A) - y^\delta) \\ 
\begin{split}
                &= \alpha AA^*y^\delta ({y^\delta})^*  A  {\left( A^*A + \alpha I  \right)^{-3} }  \\   
                & \ \ \ +\alpha A { \left( A^*A + \alpha I  \right)^{-3} } A^* y^\delta  ({y^\delta})^* A \\
                & \ \ \ -\alpha {y^\delta} ({y^\delta})^*   A{ \left( A^*A + \alpha I  \right)^{-2} }.
\end{split}
\end{align}
\label{lemma:derivative_fb}
\end{lemma}

This expression nicely collapses if ${y^\delta} ({y^\delta})^* $ commutes with $AA^*$. For illustration we assume the rather unrealistic case that $x^+=u$, where $u$ is a singular function for $A$ with singular value $\sigma$. The dual singular function is denoted by $v$, i.e.\ $Au=\sigma v$ and $A^* v= \sigma u$ and we further assume, that the measurement noise in $y^\delta$ is in the direction of this singular function, i.e., $y^\delta = (\sigma + \delta) v$, see Figure \ref{fig:example}. In this case, the problem is indeed one-dimensional and we obtain an iteration restricted to the span of $u$, resp. the span of $v$.

\begin{lemma}
The setting described above yields the following gradient step for the functional in (\ref{eq:example_functional}):
\begin{equation}
\label{eq:B_update}
B^{\ell+1} = B^\ell -  c_\ell v u^*
\end{equation}
with
\begin{equation*}
    c_\ell=c(\alpha, \delta, \sigma, \eta)=\eta   \sigma (\sigma + \delta)^2 (\alpha + \beta_\ell^2  -\sigma \beta_\ell) \frac{\beta_\ell^2 - \alpha}{(\beta_\ell^2 + \alpha)^3},
\end{equation*}
and the iteration (\ref{eq:B_update}) in-turn results in the sequence $x(B^{\ell})$ with the unique attractive stationary point

\begin{equation}
    x = 
    \begin{cases}
      \frac{1}{2\sqrt{\alpha}}(\sigma + \delta) u, & \sigma < 2 \sqrt{\alpha}\\
      \frac{1}{\sigma}(\sigma + \delta) u, & \text{otherwise.}
    \end{cases}
    \label{eq:example_result}
\end{equation}

\label{lemma:example_l2}
\end{lemma}

For comparison, the classical Tikhonov regularization would yield $\frac{\sigma}{\sigma^2 + \alpha}(\sigma + \delta ) u$. This is depicted in Figure~\ref{fig:example_reconstructions}.
 
\begin{figure}
    \centering
    \includegraphics[width=0.5\textwidth]{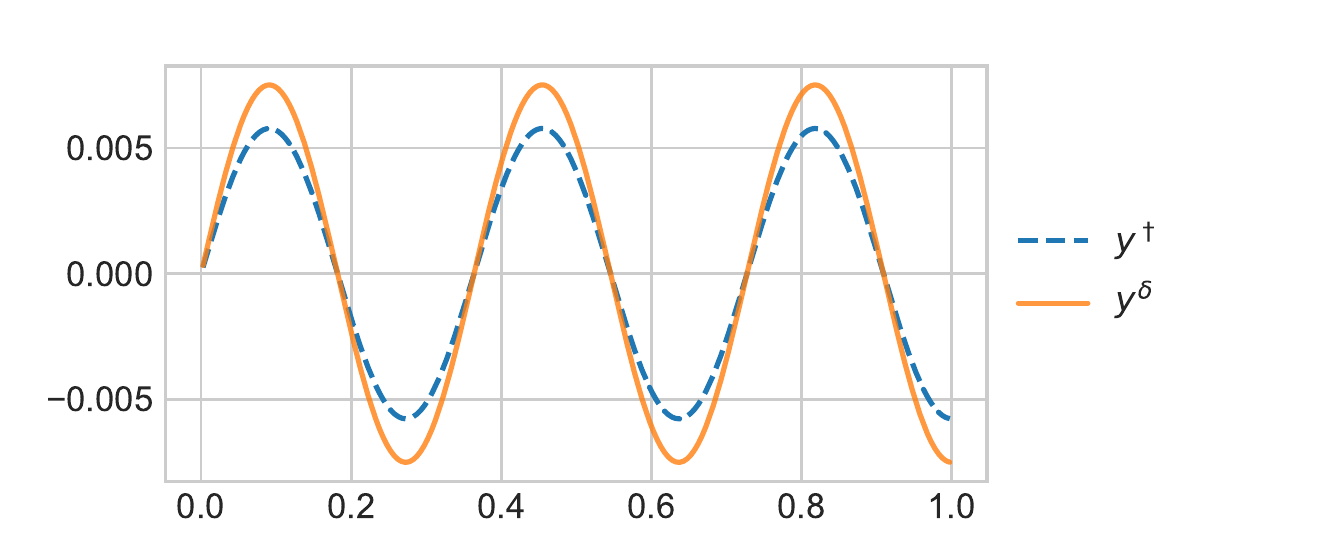}
    \caption{Example of $y^\delta = (\sigma + \delta) v$ where $v$ is a singular function of $A$ (integral operator).}
    \label{fig:example}
\end{figure}

\begin{figure}
    \centering
    \includegraphics[width=0.5\textwidth]{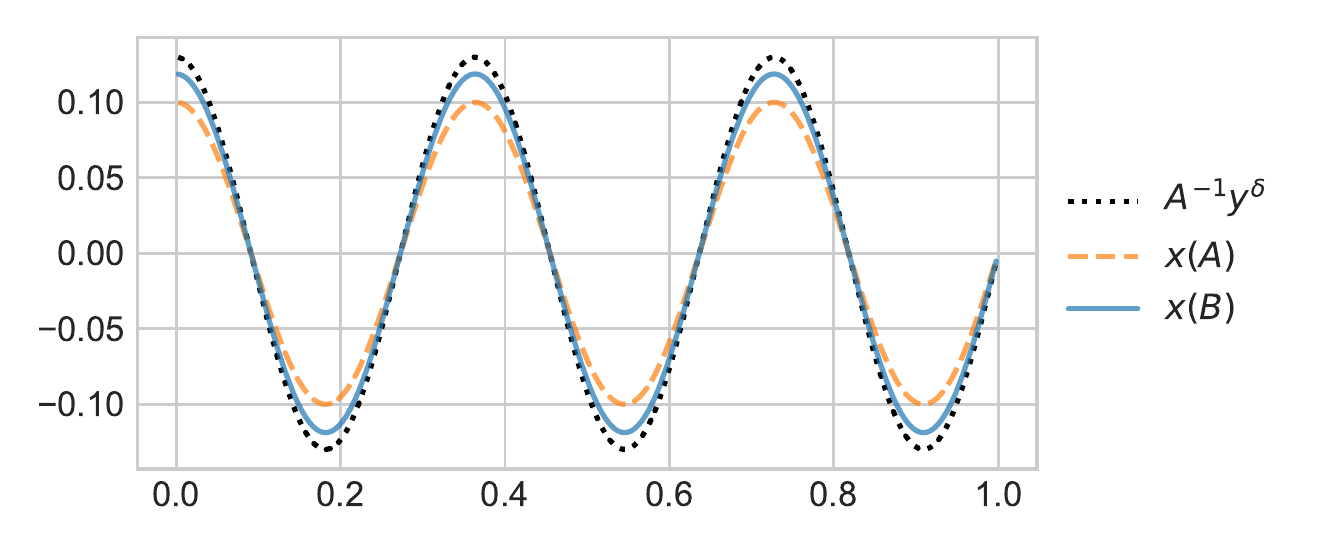}
    \caption{Comparisson of the Tikhonov reconstruction (orange broken line), the result obtained in (\ref{eq:example_result}) (blue continuous line) and the direct inverse. In this example we considered $\alpha=10^{-3}$.}
    \label{fig:example_reconstructions}
\end{figure}

\subsubsection{Constrained system of singular functions}
In the previous example we showed that if we do gradient descent starting from $B_0=A$ and assume the rather simple case $y^\delta  = (\sigma + \delta) v$, we obtain the iteration $B^{\ell+1} = B^\ell -  c_\ell v u^*$, i.e. $B^{\ell+1}$ has the same singular functions as $A$ and only one of the singular values is different.

We now analyze the optimization from a different perspective. Namely, we focus on finding directly a minimizer of (\ref{eq:adip}) for a general $y^\delta \in Y$, however, we restrict $B$ to be an operator such that $B^\ast B$ commutes with $A^\ast A$, i.e. $A$ and $B$ share a common system of singular functions. Hence, $B$ has the following representation.
\begin{equation}
B=\sum_i \beta_i v_i u_i^*, \quad  \ \beta_i \in \mathbb{R}_+\cup \{0\},
\end{equation}
where $\{u_i, \sigma_i, v_i\}$ is the singular value decomposition of $A$. That means we restrict the problem to finding optimal singular values $\beta_i$ for $B$. In this case we show that a global minimizer exists and that it has interesting properties.

\begin{theorem} For any $y^\delta \in Y$ there exist a global minimizer (in the constrained singular functions setting) of (\ref{eq:adip}) given by $B_\alpha=\sum \beta_i^\alpha v_i u_i^*$ with
\begin{equation}
\beta_i^\alpha(\sigma) = 
    \begin{cases}
       \frac{\sigma_i}{2} + \sqrt{\frac{\sigma_i^2}{4} - \alpha} & \quad \sigma \geq 2\sqrt{\alpha}\\
       \sqrt{\alpha} & \quad \sigma < 2\sqrt{\alpha} \\ 
     \end{cases}.
\end{equation}
\label{th:minimizer}
\end{theorem}

\begin{remark}
The singular values obtained in Theorem \ref{th:minimizer} match the ones obtained in the previous section for general $B$ but simple $y^\delta = (\sigma + \delta) v$.
\end{remark}

\begin{remark}
The minimizer from Theorem \ref{th:minimizer} does not depend on $y^\delta$, i.e. $\forall: y^\delta \in Y$ it holds that $B_\alpha$ is a minimizer of (\ref{eq:adip}). The solution to the inverse problem does still depend on $y^\delta$ since

\begin{equation}
x(B_\alpha) = \argmin_x\frac{1}{2} \| B_\alpha x - y^\delta\|^2  + \alpha R(x).
\end{equation}

\end{remark}

\begin{remark}
In the original DIP approach, some of the parameters of the network may be similar for different $y^\delta$, for example, the parameters of the first layers of the encoder part of the UNet. Other parameters may strongly depend on $y^\delta$. In this particular case of the analytic deep prior (constrained system of singular functions) we have a explicit separation of which parameters ($b = \lambda B^* y^\delta$) depend on $y^\delta$  and which do not  ($W =I - \lambda B^* B$).
\end{remark}

From now on we consider the notation $x(B,\,y^\delta)$ to incorporate the dependency of $x(B)$ on $y^\delta$.
Following the classical  filter theory for order optimal regularization schemes, \cite{louis,rieder,engl}, we obtain the following theorem.  
\begin{theorem}
The pseudo inverse $K_\alpha: Y \to X$ defined as
\begin{equation}
K_\alpha(y^\delta) := x(B_\alpha,\, y^\delta)
\end{equation} 
is an order optimal regularization method given by the filter functions
\begin{equation}
    F_\alpha(\sigma) = 
    \begin{cases}
       1 & \quad \sigma \geq 2\sqrt{\alpha}\\
       \frac{\sigma}{2\sqrt{\alpha}} & \quad \sigma < 2\sqrt{\alpha} \\ 
     \end{cases}.
     \label{eq:soft_filter}
\end{equation}
\label{th:order_optimal}
\end{theorem}

The regularized pseudoinverse $K_\alpha$ is quite similar to the Truncated Singular Value Decomposition (TSVD) but is a softer version because it does not have a jump (see Fig. \ref{fig:filters}). We call this method \textit{Soft TSVD}. 

The disadvantage of Tikhonov, in this case, is that it damps all singular values, and the disadvantage of TSVD is that it throws away all the information related to small singular values. On the other hand, the Soft TSVD does not damp the higher singular values (similar to TSVD) and does not throw away the information related to smaller singular values but does damp it (similar to Tikhonov). For a comparison of the filter functions, see Table~\ref{table:filter_functions}. Moreover, what is interesting is how this method comes out from Def. \ref{def:adip}, which is stated in terms of the Tikhonov pseudoinverse, and that the optimal singular values do not depend on $y^\delta$.

\begin{figure}[ht]
    \centering
    \includegraphics[width=0.37\textwidth]{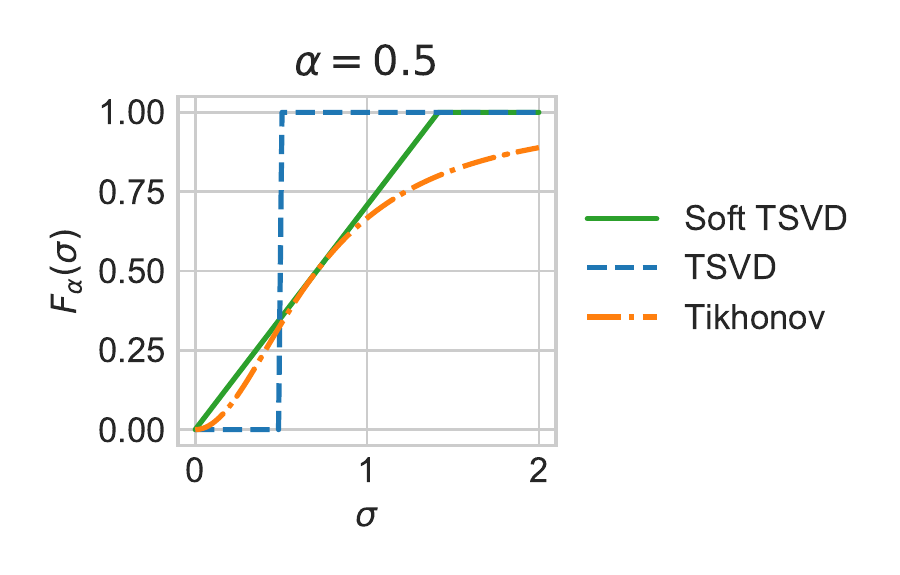}    
    \caption{Filter response of TSVD, Tikhonov and the Soft TSVD.}
    \label{fig:filters}
\end{figure}

At this point the relation to the original DIP approach becomes more abstract. We considered a simplified network architecture where all layers share the same weights that comes from an iterative algorithm for solving inverse problems. That means, we let the solution to the original inverse problem be the solution of another problem with different operator $B$. The DIP approach in this case is transformed to finding an optimal $B$ and allows us to do the analysis in the functional analysis setting. What we learn from the previous results is that we can establish interesting connections between the DIP approach and the classical Inverse Problems theory. This is important because it shows that deep inverse priors can be used to solve really ill-posed inverse problems.

\begin{remark}
\label{rk:dip_difference}
In the original DIP the input $z$ to the network is chosen arbitrarily and is of minor importance. However, once the weights have been trained for a given $y^\delta$, $z$ cannot be changed because it would affect the output of the network, i.e. it would change the obtained reconstruction. In the analytic deep prior the input to the unrolled proximal gradient method is completely irrelevant (assuming an infinite number of layers). After finding the ``weights" $B$ a different input will still produce the same solution $\hat{x} = x(B) = \varphi_\Theta(z)$.
\end{remark}

Remark \ref{rk:dip_difference} tells us that there is still a gap between the original DIP and the analytic one. This was expected because of the obvious trivialization of the network architecture but serves as motivation for further research.

\begin{table}
\centering
\begin{tabular}{llll} 
    \hline 
    Method  & $F_\alpha(\sigma)$ & $\gamma$ & $\nu$ \\
    \hline
\\
    Tikhonov & $\frac{\sigma^2}{\sigma^2+\alpha}$ & $1/2$ & $2 > \nu > 0$\\

\\

    TSVD & $\begin{cases}
       1 & \ \sigma \geq \alpha\\
       0 & \ \sigma < \alpha \\ 
     \end{cases}$ & $1$ & $\nu > 0$ \\

\\
    
    Soft TSVD & $\begin{cases}
       1 & \ \sigma \geq 2\sqrt{\alpha}\\
       \frac{\sigma}{2\sqrt{\alpha}} & \ \sigma < 2\sqrt{\alpha} \\ 
     \end{cases}$ & $1/2$ & $\nu > 0$ \\
\\
\hline

\end{tabular}
\caption{Values of $\nu$ for which TSVD, Tikhonov and the Soft TSVD are order optimal. For more details see the Proof of Theorem \ref{th:order_optimal} in Appendix II.}
\label{table:filter_functions}
\end{table}

\subsection{Numerical experiments}
\label{sec:numerics}

We now use the analytic deep inverse prior approach for solving an inverse problem with the following integration operator $A:~L^2\left(\left[0,1\right]\right)~\rightarrow ~L^2\left(\left[0,1\right]\right)$
\begin{equation}
\label{integration_operator}
\left(Ax\right)(t) = \int_0^{t}x(s)\, \text{d}s.
\end{equation}
$A$ is a linear and compact operator, hence the inverse problem is ill-posed. Let $A_n\in \mathbb{R}^{n \times n}$ be a discretization of $A$ and $x^\dagger \in \mathbb{R}^n$ to be one of its discretized singular vectors $u$. We set the noisy data ${y^\delta = A_n x^\dagger + \delta\tau}$ with ${\tau \sim \mbox{Normal}(0,\mathbb{1}_n)}$, see Figure~\ref{fig:first_example}. A more general example, i.e where $x^\dagger$ is not restricted to be a singular function, is also included (Fig. \ref{fig:extra_example}).

We aim at recovering $x^\dagger$ from $y^\delta$ considering the setting established in Def.~\ref{def:adip} for ${R(\cdot)=\frac{1}{2}\|\cdot\|^2}$. That means that the solution $x$ is para\-metrized by the operator $B$. Solving the inverse problem is now equivalent to finding optimal $B$ that minimizes the loss function (\ref{basicloss}) for the single data point $(z, y^\delta)$. 

To find such a $B$, we go back to the DIP and the neural network approach. We write $x(B)$ as the output of the network $\varphi_\Theta$ defined in \eqref{eq:network_1} with some randomly initialized input $z$. We optimize with respect to $B$, which is a matrix in the discretized setting, and obtain a minimizer $B_{\text{opt}}$ of (\ref{basicloss}). For more details, please refer to Appendix~III.

\begin{figure}
\centering
  \includegraphics[width=0.5\textwidth]{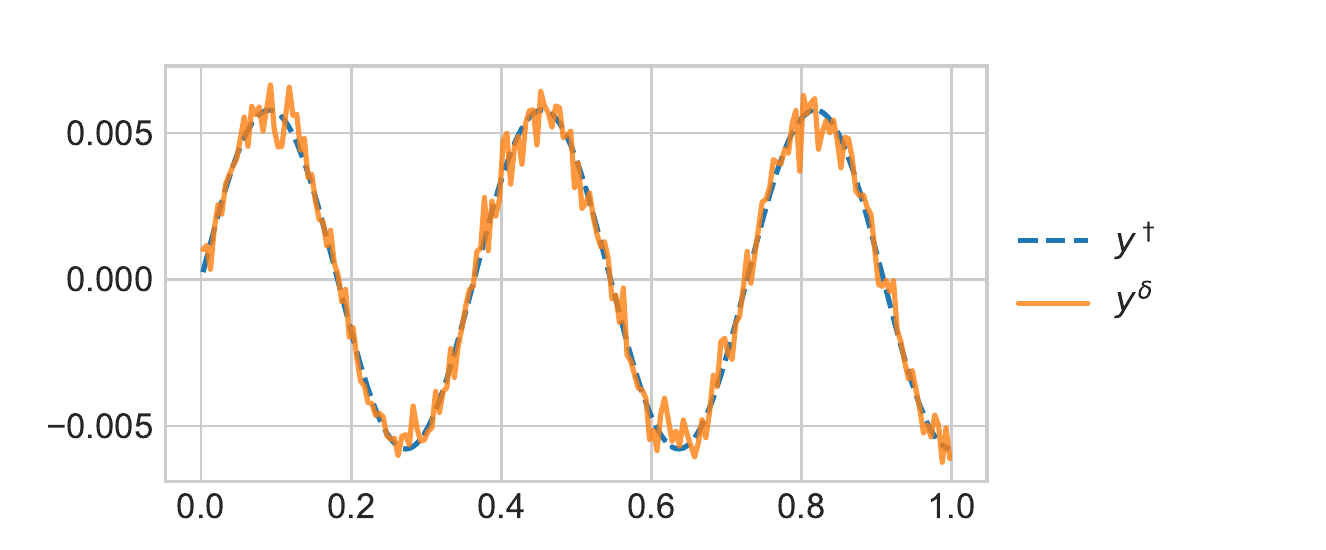}
  \caption{Example of $y^\delta$ for $x^\dagger = u$ (singular function) with a SNR of $17.06\,\text{db}$.}
  \label{fig:first_example}
\end{figure}

\begin{figure}[ht]
  \includegraphics[width=0.5\textwidth]{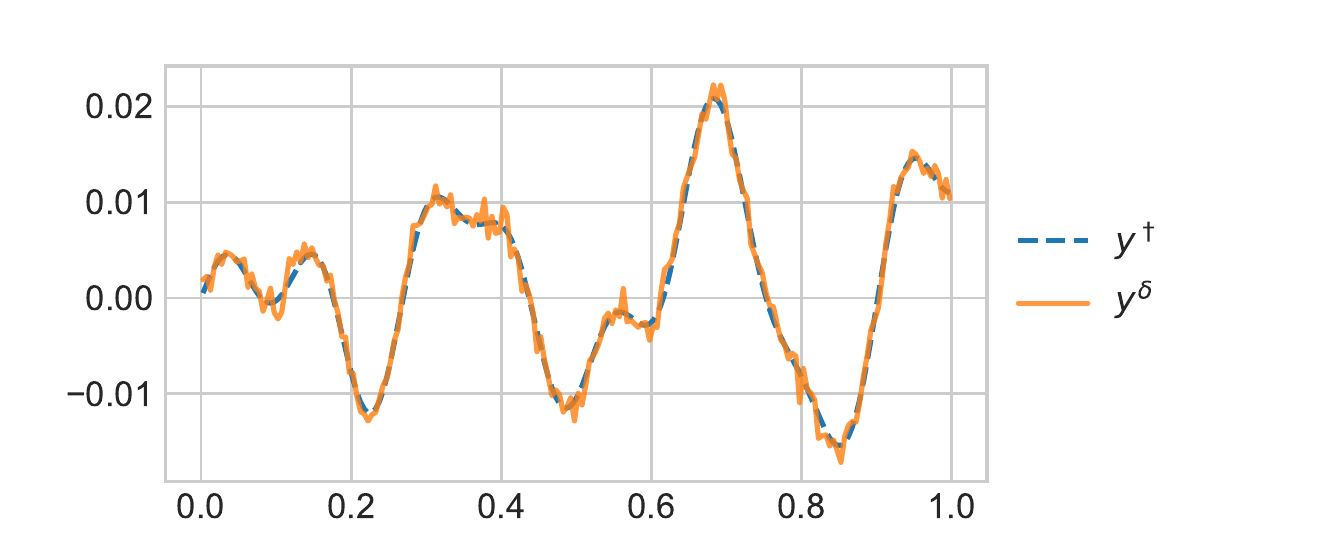}
  \caption{Example of a more general $y^\delta$ with a SNR of $18.97\,\text{db}$.}
  \label{fig:extra_example}
\end{figure}

In Figure~\ref{fig:reconstructions} we show some reconstruction results. The first plot of each row contains the true solution $x^\dagger,$ the standard Ti\-kho\-nov solution $x(A)$ and the reconstruction obtained with the analytic deep inverse approach $x(B_{\text{opt}})$ after $B$ converged. For each case we provide additional plots depicting:
\begin{itemize}
    \item The true error of the network's output $x(B)$ after each update of $B$ in a logarithmic scale. 
    \item The squared Frobenius norm of $B_k-B_{k+1}$ after each update of $B$.
    \item The matrix $B_{\text{opt}}$.
\end{itemize}

For all choices of $\alpha$ the training of $B$ converges to a matrix $B_{\text{opt}}$, such that $x(B_{\text{opt}})$ has a smaller true error than $x(A)$. In the third plot of each row, one can check that $B$ indeed converges to some matrix $B_{\text{opt}}$, which is shown in the last plot. The networks were trained using gradient descent with $0.05$ as learning rate.

The theoretical findings of the previous subsections allow us to compute, either the exact update \eqref{eq:B_update} for $B$ in the rather unrealistic case that $y^\delta = (\sigma + \delta) v$ , or the exact solution $x(B_\alpha, y^\delta)$ if we restrict $B$ to have the same system of singular functions as $A$ (Theorem \ref{th:minimizer}). In the numerical experiments we do not consider any of these restrictions and therefore we cannot directly apply our theoretical results. Instead we implement the network approach (see Appendix~III) to be able to find $B_{\text{opt}}$ in a more general scenario. Nevertheless, as it can be observed in the last plot of each row in Figure~\ref{fig:reconstructions}, $B_{\text{opt}}$ contains some patterns that reflect, to some extent, that $B$ keeps the same singular system but with different singular values. Namely, $B$ is updated in a similar way as in \eqref{eq:B_update}. With the current implementation we could also use more complex regularization functionals $R$, in order to reduce the gap between our analytic approach and the original DIP. This is also a motivation for further research.

\begin{figure*}
  \includegraphics[width=\textwidth]{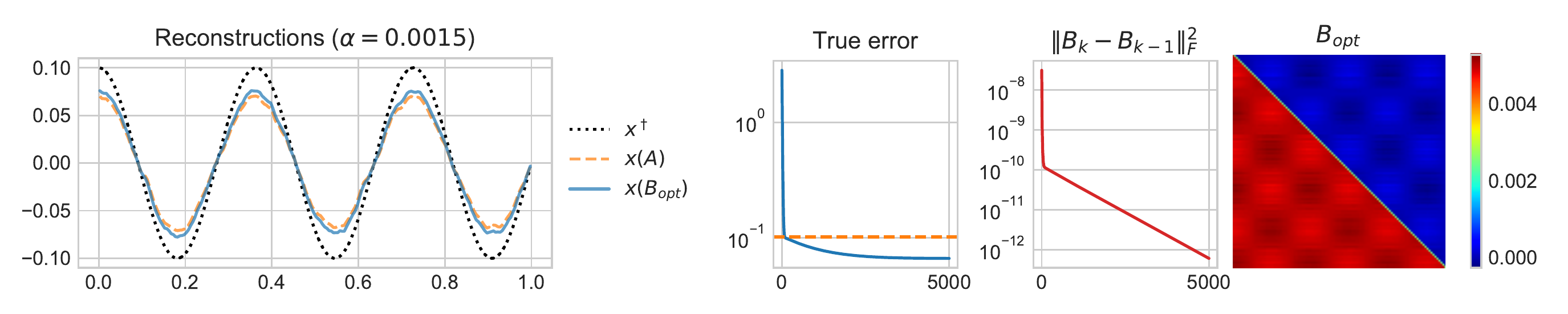}
  \includegraphics[width=\textwidth]{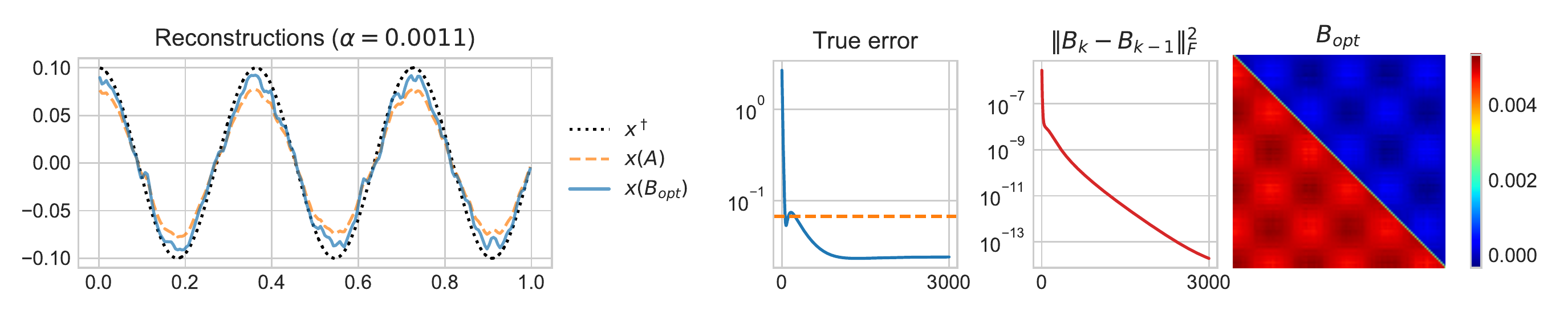}
  \includegraphics[width=\textwidth]{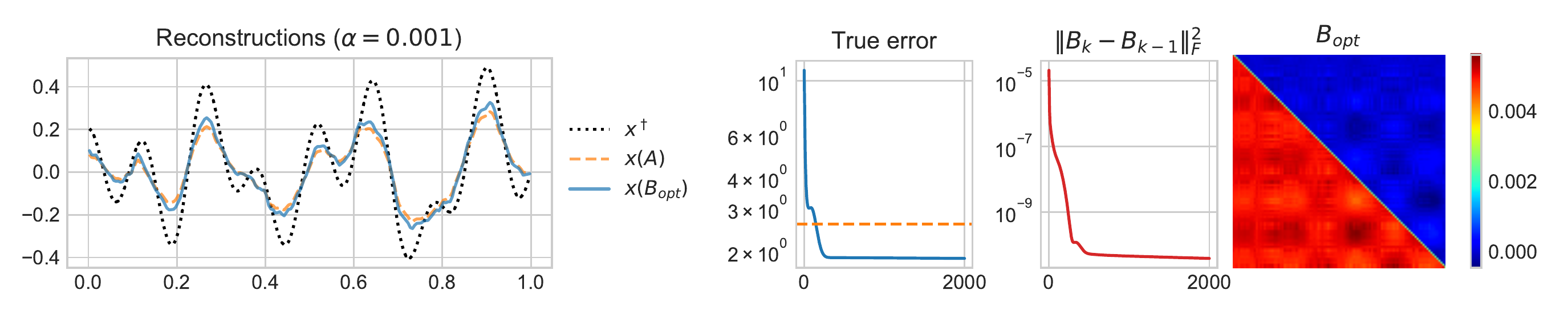}
  \includegraphics[width=\textwidth]{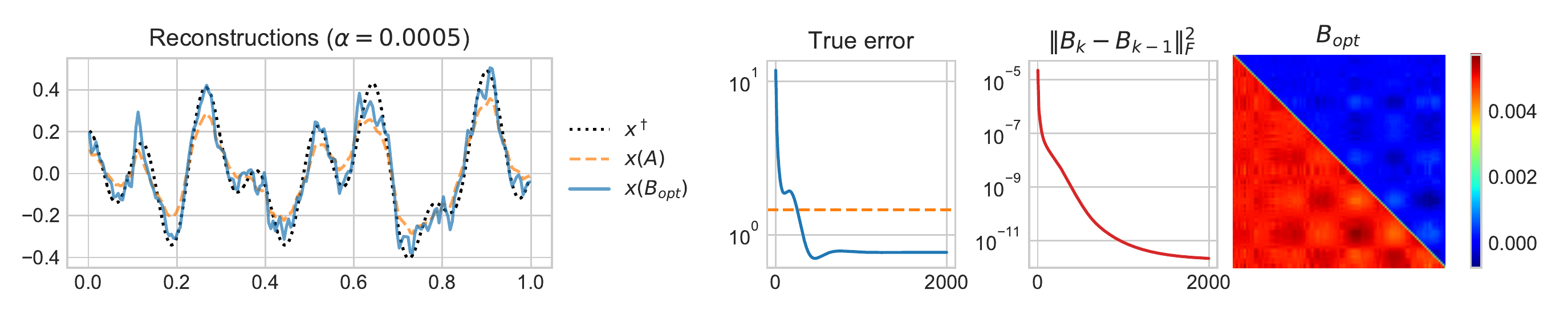}
  \centering
  \caption{Reconstructions corresponding to $y^\delta$ as in Figure~\ref{fig:first_example} (first and second row) and Figure~\ref{fig:extra_example} (third and fourth row) for different values of $\alpha$. The broken line in the second plot of each row indicates the true error of the standard Tikhonov solution $x(A)$. The horizontal axis in the second and third plots indicates the number of weights updates.}
  \label{fig:reconstructions}
\end{figure*}

\section{Summary and conclusion}
In this paper, we investigated the concept of deep inverse priors / regularization by architecture. This approach neither requires massive amounts of ground truth / surrogate data, nor pretrained models / transfer learning. The method is based on a single measurement. We started by giving different qualitative interpretations of what regularization is and specifically how regularization by architecture fits into this context.

We followed up with the introduction of the analytic deep prior by explicitly showing how unrolled proximal gradient architectures, allow for a somewhat transparent regularization by architecture. Specifically, we showed that their results can be interpreted as solutions of optimized Tikhonov functionals and proved precise equivalences to regularization techniques. We further investigated this point of view with an academic example, where we implemented the analytic deep inverse prior and tested its numerical applicability. The results confirmed our theoretical findings and showed promising results.

There is obviously, like in deep learning in general, much work to be done in order to have a good understanding of deep inverse priors, but we see much potential in the idea of using deep architectures to regularize inverse problems; especially since an enormous part of the deep learning community is already concerned with the understanding of deep architectures.

\newpage

\begin{acknowledgements}

S. Dittmer and T. Kluth acknowledge the support by the Deutsche Forschungsgemeinschaft (DFG) within the framework of GRK 2224/1 ``Pi 3 : Parameter
Identification—Analysis, Algorithms, Applications''.

D. Otero Baguer acknowledges the financial support by the Deut\-sche Forschungsgemeinschaft (DFG, German Research Foundation) –- Projektnummer 276397488 –- SFB 1232, sub-project ``P02-Heuristic, Statistical and Analytical Experimental Design''.

Peter Maass acknowledges funding by the European Union's Horizon 2020 research and innovation programme under the Marie Sklo\-dow\-ska-Curie Grant Agreement No. 765374, sub-project 'Data driven model adaptations of coil sensitivities in MR systems'.

The authors would also like to  acknowledge the anonymous reviewers for their helpful comments and suggestions.

\end{acknowledgements}

\newpage

\bibliographystyle{spmpsci}
\bibliography{w2}

\clearpage

\section*{Appendix I: A reminder on minimization of Tikhonov functionals and the LISTA approach}
\label{ap:reminder}

In this section we consider only linear operators $A$ and we review the well known theory for the Iterative Soft Shrinkage Algorithm (ISTA) as well as the slightly more general Proximal Gradient (PG) \cite{combettes2005splitting,nesterov04convex} method for minimizing Tikhonov functionals of the type
\begin{equation}
J(x)= \frac{1}{2} \| A x - y^\delta\|^2 + \alpha R(x).
\end{equation}

We recapitulate the main steps in deriving ISTA and PG, as far as we need it for our motivation. The necessary first-order condition for a minimizer is given by
\begin{equation}
0 \in A^*(Ax-y^\delta) + \alpha \partial R(x).
\end{equation}

Multiplying with an arbitrary real positive number $\lambda$ and adding $x$ plus rearranging yields
\begin{equation}
x - \lambda A^*(Ax-y^\delta) \in x + \lambda \alpha \partial R(x).
\end{equation}
For convex $R$, the term of the right hand side is inverted by the (single valued) proximal mapping of $\lambda \alpha R$, which yields
\begin{equation}
\prox_{\lambda \alpha R} \left(x - \lambda A^*(Ax-y^\delta) \right) = x. 
\end{equation}

Hence this is a fixed point condition, which is a necessary condition for all minimizers of $J$. Turning the fixed point condition into an iteration scheme yields the PG meth\-od
\begin{eqnarray} \label{ista}
x^{k+1} &=& \prox_{\lambda \alpha R} \left(x^k - \lambda A^*(Ax^k-y^\delta) \right) \\
&=&  \prox_{\lambda \alpha R} \left( (I - \lambda A^*A)x^k + \lambda A^*y^\delta \right) \ .
\end{eqnarray}

This structure is also the motivation for LISTA \cite{lista} approaches where fully connected networks with $L$ internal layers of identical size are used. Moreover, in some versions of LISTA, the affine maps between the layers are assumed to be identical. The values at the $k$-th layer are denoted by $x^k$,
hence,
\begin{equation}
x^{k+1} = \phi \left(W x^k + b\right).
\end{equation}
\noindent
LISTA then trains $(W,b)$ on some given training data. More precisely, it trains two matrices $W=I - \lambda A^*A$ and $S=\lambda A^*$ such that
\begin{equation}
x^{k+1} = \phi \left(W x^k + S y^\delta\right).
\end{equation}

This derivation can be rephrased as follows.

\begin{lemma}
Let $\varphi_\Theta$, $\Theta = (W, b)$, denote a fully connected network with input $x^0$ and $L$-internal layers. Further assume, that the activation function is identical to a proximal mapping for a convex functional $\lambda \alpha R: X \rightarrow I\!\!R$. Assume $W$ is restricted, such that $I-W$ is positive definite, i.e., there exists a matrix $B$ such that
\begin{equation}
I - W = \lambda B^* B.
\end{equation}
Furthermore, we assume that the bias term is fixed as $b = \lambda B^*y^\delta $. Then $\varphi_\Theta(z)$ is the $L$-th iterate of an ISTA scheme with starting value $x^0=z$ for minimizing
\begin{equation}
J_B(x)= \frac{1}{2} \| B x - y^\delta\|^2 + \alpha R(x).
\end{equation}
\end{lemma} 

\begin{proof} 
Follows directly from equation (\ref{ista}).
\end{proof}

\section*{Appendix II: Proofs}
\label{ap:proofs}
\subsection*{\textit{Proof of Lemma \ref{lemma:derivative}}}
$F$ is a functional which maps operators $B$ to real numbers, hence, its derivative is given by
\[\partial F (B) = \left[\partial x (B)^*\right]A^*(Ax(B) - y^\delta),\]
which follows from classical variational cal\-cu\-lus, see, e.g.\ \cite{engl}. The derivative of $x(B)$ with respect to $B$ can be computed using the fix point condition for a minimizer of $J_B$, namely
\[\prox_{\lambda \alpha R} \left(x(B) - \lambda B^*(Bx-y^\delta)\right) - x(B) = 0,\]
which is equivalent to
\[\psi(x(B),\,B)=0.\]
We apply the implicit function theorem and obtain the derivative 
\begin{equation}
\partial x(B)=- \psi_x(x(B), B)^{-1} \psi_B(x(B), B).
\end{equation}
Combining $\partial F (B) $ with $\partial x(B)$ yields the required result.$\square$

\subsection*{\textit{Proof of Lemma \ref{lemma:derivative_fb}}}
We start with the the explicit description of the iteration
\begin{equation}
    B^{\ell+1}= B^\ell - \eta \partial F (B^\ell)
\end{equation}
with
\begin{equation}
    \partial F (B) = \partial x (B)^*A^*(Ax(B) - y^\delta).
\end{equation}

The derivative of $x(B)$ with respect to $B$ is a linear map $\partial x(B): {\cal L}(X,Y) \rightarrow X.$ For $\delta B \in {\cal L}(X,Y)$ we obtain
\begin{align}
    \begin{split}
        \partial x (B) (\delta B) = &- \left( B^*B + \alpha I \right)^{-2} \left( \delta B^* B + B^*\delta B \right) B^* y^\delta \\
        &+ \left( B^*B + \alpha I\right)^{-1} \delta B^* y^\delta.
    \end{split}
\end{align}
The adjoint operator is a mapping from $X$ to ${\cal L} (X,Y)$, which can be derived from the defining relation
\begin{equation}
\inner{\partial x(B)(\delta B)}{z}_X = \inner{\delta B}{\left[ \partial x(B)  \right]^* z}_{{\cal L}(X,Y)}   \ .  
\end{equation}
Hence,
\begin{align}
\begin{split}
\left[ \partial x(B) \right]^* z =  & -BB^*y^\delta z^*   {\left( B^*B + \alpha I  \right)^{-2} } \\     
                                    & -B{\left( B^*B + \alpha I  \right)^{-2}}
z ({y^\delta})^* B \\
                                    & + {y^\delta} z^*   { \left( B^*B + \alpha
I  \right)^{-1} }.    
\end{split}
\end{align}
Here, ${y^\delta} z^* \in {\cal L}(X,Y)$ denotes a linear map, which maps an $x \in X$ to $\inner{z}{x}_X \ y^\delta$. 

First of all, we now aim at determining explicitly $\partial F (B)$ at the starting point of our iteration, i.e., with $B^0=A$.

From this follows the rather lengthy expression 
\begin{align}
\partial F (A)  &=\partial x (A)^*A^*(Ax(A) - y^\delta) \\ 
\begin{split}
                &= \alpha AA^*y^\delta ({y^\delta})^*  A  {\left( A^*A + \alpha I  \right)^{-3} }  \\   
                & \ \ \ +\alpha A { \left( A^*A + \alpha I  \right)^{-3} } A^* y^\delta  ({y^\delta})^* A \\
                & \ \ \ -\alpha {y^\delta} ({y^\delta})^*   A{ \left( A^*A + \alpha I  \right)^{-2} }.
\end{split}
\end{align}
This enables us to compute the update
\begin{equation}
B_1= A - \eta \partial  F (A)
\end{equation}
as well as the output of the analytic deep prior approach after one iteration of updating $B$ (assuming a suitably chosen $\eta$)
\begin{equation}
x(B_1)=(B_1^*B_1+\alpha I)^{-1}B_1^*y^\delta.
\end{equation}

\subsection*{\textit{Proof of Lemma \ref{lemma:example_l2}}}

A lengthy computation exploiting $B^0=A$ and $ \beta_0=\sigma$  shows that  the singular value $\beta_\ell$ of $u$ in the spectral decomposition of $B^\ell$ obeys the iteration
\begin{equation}
    \beta_{\ell+1}=  \beta_\ell - \eta   \sigma (\sigma + \delta)^2 (\alpha + \beta_\ell^2  -\sigma \beta_\ell) \frac{\beta_\ell^2 - \alpha}{(\beta_\ell^2 + \alpha)^3},
    \label{eq:fixedpointiteration}
\end{equation}
 i.e.,
\begin{equation}
\label{B_update}
B^{\ell+1} = B^\ell -  c_\ell v u^*
\end{equation}
with 
\[c_\ell=c(\alpha, \delta, \sigma, \eta)=\eta   \sigma (\sigma + \delta)^2 (\alpha + \beta_\ell^2  -\sigma \beta_\ell) \frac{\beta_\ell^2 - \alpha}{(\beta_\ell^2 + \alpha)^3} \ .\]

We will now consider the stability of the fixed points of the sequence $x(B^\ell)$, i.e., we will analyze the fixed points of the iteration described in (\ref{eq:fixedpointiteration}), that is,
\begin{equation}
\beta_{\ell+1}=  \beta_\ell - c(\beta_\ell),
\end{equation}
where
\begin{equation}
    c(\beta) = \eta \sigma (\sigma + \delta)^2 (\alpha + \beta^2  -\sigma \beta) \frac{\beta^2 - \alpha}{(\beta^2 + \alpha)^3}.
\end{equation}
This iteration in-turn gives you via the Tikhonov filter function, the sequence
\begin{equation}
        x(\beta_\ell) = \frac{\beta_\ell}{\beta_\ell^2 + \alpha}(\sigma + \delta ) u
\end{equation}
of reconstructions.
To find the fixed points of the iteration, we analyze the real roots of $c$,
which are
\begin{itemize}
    \item $\beta^{(1)} = \sqrt{\alpha}$,
    \item $\beta^{(2)} = -\sqrt{\alpha}$,
    \item $\beta^{(3)} = \frac{\sigma}{2} + \sqrt{\frac{\sigma^2}{4} - \alpha}$, for $\sigma \ge 2 \sqrt{\alpha}$ and
    \item $\beta^{(4)} = \frac{\sigma}{2} - \sqrt{\frac{\sigma^2}{4} - \alpha}$, for $\sigma \ge 2 \sqrt{\alpha}$.
\end{itemize}
Simple calculations show that
\begin{itemize}
    \item $\partial_\beta c(\beta^{(1)})  
    \begin{cases}
      >0, & \sigma < 2 \sqrt{\alpha}\\
      \le0, & \text{otherwise.}
    \end{cases}$
    \item  $\partial_\beta c(\beta^{(2)}) > 0$
    \item  $\partial_\beta c(\beta^{(3)}) > 0$, for $\sigma \ge 2 \sqrt{\alpha}$ and 
    \item  $\partial_\beta c(\beta^{(4)}) > 0$ for $\sigma \ge 2 \sqrt{\alpha}$.
\end{itemize}
This leads to the single attractive fixed point $\beta^{(1)}$ for $\sigma < 2\sqrt{\alpha}$ and the two attractive fixed points $\beta^{(3)}$ and $\beta^{(4)}$ otherwise. Since,
\begin{equation}
    x(\beta^{(3)}) = x(\beta^{(4)}),
\end{equation}
we therefore have a unique reconstruction, namely
\begin{equation}
    x = 
    \begin{cases}
      \frac{1}{2\sqrt{\alpha}}(\sigma + \delta) u, & \sigma < 2 \sqrt{\alpha}\\
      \frac{1}{\sigma}(\sigma + \delta) u, & \text{otherwise.}
    \end{cases}
\end{equation}
$\square$

\subsection*{\textit{Proof of Theorem \ref{th:minimizer}}}
Let $B=\sum \beta_i v_i u_i^*$. We want to find $\{\beta_i\}$ to minimize
\begin{equation}
F(B) = \Vert A x(B, y^\delta) -y^\delta\Vert^2.
\label{eq:jb}
\end{equation}
The Tikhonov solution is given by
\begin{equation}
    x(B) = \sum \frac{\beta_i}{\beta_i^2 + \alpha}\inner{y^\delta}{v_i}u_i,
\end{equation}
the result of applying the operator $A$ to $x(B)$ is
\begin{equation}
    Ax(B) = \sum \frac{\sigma_i\beta_i}{\beta_i^2 + \alpha}\inner{y^\delta}{v_i}v_i
    \label{eq:aty_delta}
\end{equation}
and
\begin{equation}
    y^\delta = \sum \inner{y^\delta}{v_i}v_i.
    \label{eq:y_delta}
\end{equation}
Inserting (\ref{eq:aty_delta}) and (\ref{eq:y_delta}) in (\ref{eq:jb}) yields
\begin{equation}
    F(B) = \sum \left|(\frac{\sigma_i\beta_i}{\beta_i^2 + \alpha}-1)\inner{y^\delta}{v_i}\right|^2 .
    \label{eq:jb_sum}
\end{equation}
In order to minimize $F(B)$, we should set $\frac{\sigma\beta_i}{\beta_i^2 + \alpha} = 1$ which implies $\beta_i^2 - \sigma \beta_i + \alpha = 0$. The roots of the previous equation are $\beta_i=\frac{\sigma_i}{2} \pm \sqrt{\frac{\sigma_i^2}{4} - \alpha}$ and they are real only  if $\frac{\sigma_i^2}{4} \geq \alpha$. If it does not hold then $\frac{\alpha\beta_i}{\beta_i^2+\alpha} < 1$ and the optimal choice is to find its maximum value which is attained at $\beta_i=\sqrt{\alpha}$.

Therefore we set
\begin{equation}
\beta_i = 
    \begin{cases}
       \frac{\sigma_i}{2} + \sqrt{\frac{\sigma_i^2}{4} - \alpha} & \quad \sigma \geq 2\sqrt{\alpha}\\
       \sqrt{\alpha} & \quad \sigma < 2\sqrt{\alpha} \\ 
     \end{cases}
\end{equation}
and we minimize every term in the sum (\ref{eq:jb_sum}), which means we have found singular values $\{\beta_i\}$ that minimize $F(B)$.$\square$

\subsection*{\textit{Proof of Theorem \ref{th:order_optimal}}}
In order to prove that $K_\alpha$ is a proper order optimal regularization method we need to check if the corresponding filters $F_\alpha$ from (\ref{eq:soft_filter}) satisfy the three conditions of optimality \cite{louis,rieder}.

These conditions state that a filter $F_\alpha:\mathbb{R} \to \mathbb{R}$ is an order optimal regularization filter if $\exists\, \gamma,\, c_1,\,c_2,\,c_3 > 0$ such that
\begin{enumerate}
    \item $\sup_\sigma{\left|F_\alpha(\sigma)\sigma^{-1}\right|} \leq c_1\alpha^{-\gamma}$
    \item $\sup_\sigma{\left|1-F_\alpha(\sigma)\right|\sigma^{\nu}} < c_2\alpha^{\gamma\nu}$
    \item $\forall \alpha>0, \sigma>0: \left|F_\alpha(\sigma)\right| \leq c_3$
\end{enumerate}

In the following we show that they hold $\forall \nu > 0$ with $\gamma=\frac{1}{2},\, c_1=\frac{1}{2},\, c_2=2^\nu,\, c_3=1$:
\begin{enumerate}
\item[i] {If $\sigma \geq 2 \sqrt{\alpha}$
    \begin{enumerate}
        
        \item [1.]{ $\sup_\sigma{\left|F_\alpha(\sigma)\sigma^{-1}\right|} = \sup_\sigma{\left|\sigma^{-1}\right|} \leq \frac{1}{2}\alpha ^{-\frac{1}{2}}$}
        
        \item [2.]{$\sup_\sigma{\left|1-F_\alpha(\sigma)\right|\sigma^{\nu}}=0\leq \alpha ^\nu$}
        
        \item [3.]{$\forall \alpha>0, \sigma>0: \left|F_\alpha(\sigma)\right|=1$}
    
    \end{enumerate}
}

\item[ii]{ If $\sigma < 2\sqrt{\alpha}$
\begin{enumerate}
    \item [1.] {$\sup_\sigma{\left|F_\alpha(\sigma)\sigma^{-1}\right|}= \frac{1}{2}\alpha ^{-\frac{1}{2}}$}
    
    \item [2.] \(\begin{aligned}[t]
        {\sup}_\sigma{\left|1-F_\alpha(\sigma)\right|\sigma^{\nu}}&=\sup_\sigma{\left|\frac{2\sqrt{\alpha}-\sigma}{2\sqrt{\alpha}}\right|\sigma^{\nu}}\\ & \leq 2^\nu \alpha^{\frac{\nu}{2}}
    \end{aligned}\)
    
    \item [3.]{$\forall \alpha>0, \sigma>0: \left|F_\alpha(\sigma)\right|=\frac{\sigma}{2\sqrt{\alpha}} \leq 1$}
    
    \end{enumerate}
}
\end{enumerate}$\square$

\section*{Appendix III: Numerical experiments}
\label{ap:numerics}

In this section, we provide details about the implementation of the analytic deep inverse prior and the academic example. We start by discretizing the integration operator, which yields the matrix ${A_n\in\mathbb{R}^{n \times n}}$, that has $\frac{h}{2}$ on the main diagonal, $h$ everywhere under the main diagonal and $0$ above (here $h=\frac{1}{n}$). In our experiments we use $n=200$.

The analytic deep inverse prior network $\varphi^L_\Theta$ is implemented using Python and Tensorflow~\cite{tensorflow2015-whitepaper}. Initially, we create the matrix $B \in \mathbb{R}^{n \times n}$ and add $L$ fully connected layers to the network, all having the same parameters $\Theta = (W,b)$, with weight matrix $W= I- \lambda B^TB$, bias $b=\lambda B^T y^\delta$ and activation function given by the $\ell_2$ proximal operator. That means the network contains in total $4 \times 10^4$ parameters (the number of components in $B$). For the experiments shown in the paper, the input $z$ is randomly initialized with a small norm and $\lambda$ is $\frac{1}{\mu}$, where $\mu$ is the biggest eigenvalue of $A^TA$.

We follow the DIP approach and minimize \eqref{basicloss} using gradient descent. To guarantee that $\varphi^L_\Theta(z) = x(B)$ holds, the network should have thousands of layers, because of the slow convergence of the PG method. This is prohibitive from the implementation point of view. Therefore, we consider only a reduced network with a small number of layers, $L=10,$ and at each iteration we set the input of the network to be the network's output after the previous iteration. This is equivalent to adding $L$ new identical layers after each update of $B$, with 
\begin{equation}
{W_i = I- \lambda B_i^TB_i}
\end{equation}
and 
\begin{equation}
b_i=\lambda B_i^T y^\delta,
\end{equation}
where $B_i$ refers to the value of $B$ at the $i$-th iteration. After $k$ iterations, we implicitly create a network that has $(k+1)L$ layers, however, each time we update $B$, we back-propagate only through the last $L$ layers.

\begin{figure}
  \tikzset{every picture/.style={line width=0.75pt}} %set default line width to 0.75pt        

\begin{tikzpicture}[x=0.75pt,y=0.75pt,yscale=-1,xscale=1]
%uncomment if require: \path (0,413.6363525390625); %set diagram left start at 0, and has height of 413.6363525390625

%Shape: Ellipse [id:dp5303945113101574] 
\draw   (330.2,146.18) .. controls (330.2,137.68) and (336.99,130.79) .. (345.36,130.79) .. controls (353.73,130.79) and (360.52,137.68) .. (360.52,146.18) .. controls (360.52,154.68) and (353.73,161.57) .. (345.36,161.57) .. controls (336.99,161.57) and (330.2,154.68) .. (330.2,146.18) -- cycle ;
%Straight Lines [id:da7532628771309524] 
\draw    (360.52,146.18) -- (379.89,146.27) ;
\draw [shift={(381.89,146.28)}, rotate = 180.26] [fill={rgb, 255:red, 0; green, 0; blue, 0 }  ][line width=0.75]  [draw opacity=0] (8.93,-4.29) -- (0,0) -- (8.93,4.29) -- cycle    ;

%Straight Lines [id:da8449986722622227] 
\draw    (308.77,146.18) -- (328.2,146.18) ;
\draw [shift={(330.2,146.18)}, rotate = 540] [fill={rgb, 255:red, 0; green, 0; blue, 0 }  ][line width=0.75]  [draw opacity=0] (8.93,-4.29) -- (0,0) -- (8.93,4.29) -- cycle    ;

%Straight Lines [id:da8819082109390828] 
\draw    (579.69,145.91) -- (595.27,145.79) ;
\draw [shift={(597.27,145.78)}, rotate = 539.5699999999999] [fill={rgb, 255:red, 0; green, 0; blue, 0 }  ][line width=0.75]  [draw opacity=0] (8.93,-4.29) -- (0,0) -- (8.93,4.29) -- cycle    ;

%Shape: Rectangle [id:dp33460040060038065] 
\draw  [color={rgb, 255:red, 0; green, 0; blue, 0 }  ,draw opacity=1 ][dash pattern={on 4.5pt off 4.5pt}] (382.86,127.68) -- (419.39,127.68) -- (419.39,164.78) -- (382.86,164.78) -- cycle ;
%Straight Lines [id:da550507686037282] 
\draw    (420.52,145.68) -- (439.89,145.77) ;
\draw [shift={(441.89,145.78)}, rotate = 180.26] [fill={rgb, 255:red, 0; green, 0; blue, 0 }  ][line width=0.75]  [draw opacity=0] (8.93,-4.29) -- (0,0) -- (8.93,4.29) -- cycle    ;

%Shape: Rectangle [id:dp5148530867397427] 
\draw  [color={rgb, 255:red, 0; green, 0; blue, 0 }  ,draw opacity=1 ][dash pattern={on 4.5pt off 4.5pt}] (442.86,127.18) -- (479.39,127.18) -- (479.39,164.28) -- (442.86,164.28) -- cycle ;
%Straight Lines [id:da3809624811104997] 
\draw    (523.77,145.78) -- (539.89,145.78) ;
\draw [shift={(541.89,145.78)}, rotate = 180] [fill={rgb, 255:red, 0; green, 0; blue, 0 }  ][line width=0.75]  [draw opacity=0] (8.93,-4.29) -- (0,0) -- (8.93,4.29) -- cycle    ;

%Shape: Rectangle [id:dp8052721369841549] 
\draw  [color={rgb, 255:red, 0; green, 0; blue, 0 }  ,draw opacity=1 ][dash pattern={on 4.5pt off 4.5pt}] (542.86,127.18) -- (579.39,127.18) -- (579.39,164.28) -- (542.86,164.28) -- cycle ;
%Straight Lines [id:da33114710824204274] 
\draw    (480.02,146.18) -- (499.39,146.27) ;
\draw [shift={(501.39,146.28)}, rotate = 180.26] [fill={rgb, 255:red, 0; green, 0; blue, 0 }  ][line width=0.75]  [draw opacity=0] (8.93,-4.29) -- (0,0) -- (8.93,4.29) -- cycle    ;

%Shape: Brace [id:dp5976609931303454] 
\draw   (371.3,169.68) .. controls (371.31,174.35) and (373.64,176.68) .. (378.31,176.67) -- (470.06,176.46) .. controls (476.73,176.45) and (480.07,178.77) .. (480.08,183.44) .. controls (480.07,178.77) and (483.39,176.43) .. (490.06,176.41)(487.06,176.42) -- (581.81,176.2) .. controls (586.48,176.19) and (588.81,173.86) .. (588.8,169.19) ;

% Text Node
\draw (345.36,146.18) node [scale=1] [align=left] {$\displaystyle z$};
% Text Node
\draw (401.12,146.23) node [scale=1.2,] [align=left] {$\displaystyle \varphi ^{L}_{\Theta _{0}}$};
% Text Node
\draw (613.36,145.18) node [scale=1.2] [align=left] {$\displaystyle x( B)$};
% Text Node
\draw (421.46,171.85) node  [align=left] {$ $};
% Text Node
\draw (515.12,145.73) node [scale=1.2,] [align=left] {$\displaystyle ...$};
% Text Node
\draw (480.86,190.68) node  [align=left] {$\displaystyle k+1$};
% Text Node
\draw (435.53,193.35) node  [align=left] {$ $};
% Text Node
\draw (461.12,145.73) node [scale=1.2,] [align=left] {$\displaystyle \varphi ^{L}_{\Theta _{1}}$};
% Text Node
\draw (561.12,145.73) node [scale=1.2,] [align=left] {$\displaystyle \varphi ^{L}_{\Theta _{k}}$};

\end{tikzpicture}
  \centering
  \caption{The implicit network with $(k+1)L$ layers. Here $\varphi^L_{\Theta_k}$ refers to a block of $L$ identical fully connected layers with weights ${\Theta_k = (W_k,\,b_k)}$.}
  \label{network}
\end{figure}
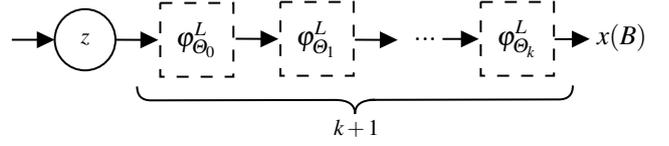

\end{document}